 \documentclass[conference, 10pt]{IEEEtran}

\IEEEoverridecommandlockouts                              

\usepackage{graphics} 
\usepackage{epsfig} 
\usepackage{mathptmx} 
\usepackage{mathrsfs}
\usepackage{times} 
\usepackage{amsmath} 

\usepackage{amsfonts}
\usepackage{amsthm}
\usepackage{amssymb}  
\usepackage{wasysym}
\usepackage{txfonts}
\usepackage{enumerate}
\usepackage{bbm}
\usepackage{bm}
\usepackage[singlelinecheck=false,font=footnotesize]{caption}
\usepackage{subcaption}
\usepackage{color}  
\usepackage[dvipsnames]{xcolor}
\usepackage{indentfirst}
\usepackage{multirow}
\usepackage{color, colortbl}
\usepackage{algorithm}
\usepackage{algpseudocode}

\usepackage{geometry}
\usepackage[
    style=ieee,
    doi=false,
    isbn=false,
    url=false,
    eprint=false,
    backend=bibtex,
    natbib=true
    ]{biblatex}

\bibliography{ref}

\usepackage{hyperref}

\newtheorem{defn}{Definition}

\newtheorem{lem}[defn]{Lemma}

\newtheorem{assum}[defn]{Assumption}

\newtheorem{thm}[defn]{Theorem}

\providecommand{\FRS}{\textcolor{black}{(\text{FRSopt})}}
\providecommand{\R}{\ensuremath \mathbb{R}}
\providecommand{\N}{\ensuremath \mathbb{N}}

\providecommand{\Later}[1]{}

\providecommand{\HIP}{SBM}

\providecommand{\naive}{na\"ive}
\providecommand{\Y}{\mathcal Y}

\providecommand{\PP}{\mathcal P}
\providecommand{\B}{\mathcal B}
\providecommand{\G}{\mathcal G}
\providecommand{\W}{\mathcal W}
\providecommand{\Q}{\mathcal Q}
\providecommand{\I}{\mathcal I}

\begin{document}

\title{Leveraging the Template and Anchor Framework for Safe, Online Robotic Gait Design}

\author{Jinsun Liu$^*$, Pengcheng Zhao$^*$, Zhenyu Gan, Matthew Johnson-Roberson, and Ram Vasudevan
\thanks{Jinsun Liu is with the Robotics Institute, University of Michigan, Ann Arbor, MI 48109 \texttt{jinsunl@umich.edu}}
\thanks{Pengcheng Zhao, Zhenyu Gan and Ram Vasudevan are with the Department of Mechanical Engineering, University of Michigan, Ann Arbor, MI 48109 \texttt{\{pczhao,ganzheny,ramv\}@umich.edu }}
\thanks{Matthew Johnson-Roberson is with the Department of Naval Architecture and Marine Engineering, University of Michigan, Ann Arbor, MI 48109 \texttt{mattjr@umich.edu}}
\thanks{This work is supported by the Ford Motor Company via the Ford-UM Alliance under award N022977 and by the Office of Naval Research under Award Number N00014-18-1-2575.}
\thanks{$*$These two authors contributed equally to this work.}
}

\maketitle

\begin{abstract}
Online control design using a high-fidelity, full-order model for a bipedal robot can be challenging due to the size of the state space of the model.
A commonly adopted solution to overcome this challenge is to approximate the full-order model (anchor) with a simplified, reduced-order model (template), while performing control synthesis.
Unfortunately it is challenging to make formal guarantees about the safety of an anchor model using a controller designed in an online fashion using a template model.
To address this problem, this paper proposes a method to generate safety-preserving controllers for anchor models by performing reachability analysis on template models while bounding the modeling error. 
This paper describes how this reachable set can be incorporated into a Model Predictive Control framework to select controllers that result in safe walking on the anchor model in an online fashion.
The method is illustrated on a 5-link RABBIT model, and is shown to allow the robot to walk safely while utilizing controllers designed in an online fashion.  
\end{abstract}

\section{Introduction}

Legged robots are an ideal system to perform locomotion on unstructured terrains.
Unfortunately designing controllers for legged systems to operate safely in such situations has proven challenging. 
To robustly traverse such environments, an ideal control synthesis technique for legged robotic systems should satisfy several requirements.
First, since sensors perceive the world with a limited horizon, any algorithm for control synthesis should operate in real-time. 
Second, since modeling contact can be challenging, any control synthesis technique should be able to accommodate model uncertainty.  
Third, since the most appropriate controller may be a function of the environment and given task, a control synthesis algorithm should optimize over as rich a family of control inputs at run-time as possible.
Finally, since falling can be costly both in time and expense, a control synthesis technique should be able to guarantee the satisfactory behavior of any constructed controller. 
As illustrated in Fig. \ref{fig:Big Pic}, this paper presents an optimization-based algorithm to design gaits for legged robotic systems while satisfying each of these requirements.

\begin{figure}[t]
\centering  
  {\includegraphics[width=1.1\columnwidth]{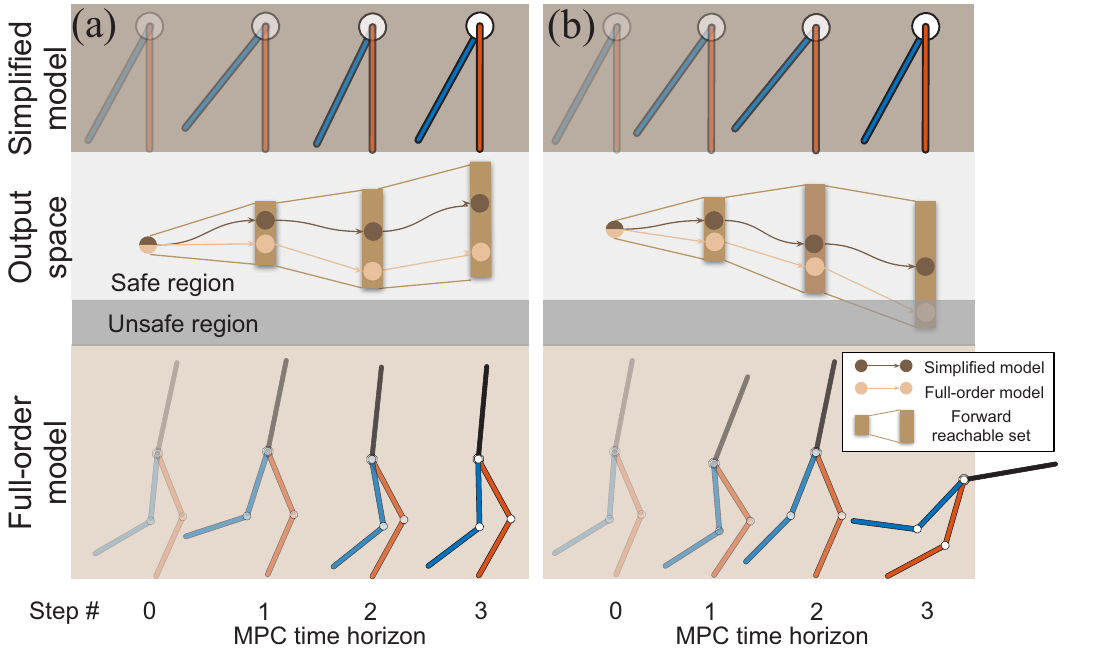}}
\caption{This paper proposes a method to design gaits that are certified to be tracked by a full-order robot model (bottom row sub-figures) for $N$-steps without falling over.
To construct this method, this paper defines a set of outputs that are functions of the state of the robot and a chosen gait (middle row sub-figures).
If the outputs associated with a particular gait satisfy a set of inequality constraints (depicted as the safe region drawn in light gray in the middle row sub-figures), then the gait is proven to be safely tracked by the legged system without falling.
Due to the high-dimensionality of the robot's dynamics, forward propagating these outputs via the robot's dynamics for $N$-steps to design a gait that is certified to be tracked safely is intractable.
To address this challenge, this paper constructs a template model (top row sub-figures) whose outputs are sufficient to predict the behavior of the anchor's outputs. 
In particular, if all of the points in a bounded neighborhood of the forward reachable set of the outputs of the template model remain within the safe region, then the anchor is certified to behave safely.
This paper illustrates how this can be incorporated into a MPC framework to design safe gaits in real-time.}
\label{fig:Big Pic}
\end{figure}

We begin by summarizing related work with an emphasis on techniques that are able to make guarantees on the safety of the designed controller.
For instance, the Zero-Moment Point approach \cite{vukobratovic1972stability} characterizes the stability of a legged robot with planar feet by defining the notion of the Zero-Moment Point and requiring that it remains within the robot's base of support. 
Though this requirement can be used to design a controller that can avoid falling at run-time, the gaits designed by the ZMP approach are static and energetically expensive \cite[Section~10.8]{kuo2007choosing,westervelt2007feedback}. 

In contrast, the Hybrid Zero Dynamics approach, which relies upon feedback linearization to drive the actuated degrees of freedom of a robot towards a lower dimensional manifold, is able to synthesize a controller which generates gaits that are more dynamic. 
Though this approach can generate safety preserving controllers for legged systems in real-time in the presence of model uncertainty \cite{ames2014rapidly,hsu2015control,nguyen2015optimal,nguyen2016exponential,nguyen2016dynamic}, it is only able to prove that the gait associated with a synthesized control is locally stable. 
As a result, it is non-trivial to switch between multiple constructed controllers while preserving any safety guarantee. 
Recent work has extended the ability of the hybrid zero dynamic approach beyond a single neighborhood of any synthesized gait \cite{motahar2016composing,veer2018safe,ames2017first,smit2019walking}.
These extensions either assume full-actuation \cite{ames2017first} or ignore the behavior of the legged system off the lower dimensional manifold \cite{motahar2016composing,veer2018safe,smit2019walking}. 

Rather than designing controllers for legged systems, other techniques have focused on characterizing the limits of safe performance by using Sums-of-Squares (SOS) optimization \cite{parrilo2000structured}.
These approaches use semi-definite programming to identify the limits of safety in the state space of a system as well as associated controllers for hybrid systems \cite{prajna,shia2014convex}.
These safe sets can take the form of \emph{reachable sets} \cite{koolen2016balance,shia2014convex} or \emph{invariant sets} in state space \cite{wieber2002stability,prajna,posa2017balancing}.
However, the representation of each of these sets in state space restricts the size of the problem that can be tackled by these approaches and as a result, these SOS-based approaches have been primarily applied to reduced models of walking robots: ranging from spring mass models \cite{zhao2017optimal}, to inverted pendulum models \cite{koolen2016balance,tang2017invariant} and  to inverted pendulum models with an offset torso mass \cite{posa2017balancing}.
Unfortunately the differences between these simple models and real robots makes it challenging to extend the safety guarantees to more realistic real-world models.

This paper addresses the shortcomings of prior work by making the following four contributions.
First, in Section \ref{subsec:outputsrabbit}, we describe a set of outputs that are functions of the state of the robot, which can be used to determine whether a particular gait can be safely tracked by a legged system without falling.
In particular, if a particular gait's outputs satisfy a set of inequality constraints that we define, then we show that the gait can be safely tracked by the legged system without falling.
To design gaits over $N$-steps that do not fall over, one could begin by forward propagating these outputs via the robot's dynamics for $N$-steps.
Unfortunately performing this computation can be intractable due to the high-dimensionality of the robot's dynamics. 
To address this challenge, our second contribution, in Section \ref{subsec:outputssimple}, leverages the anchor and template framework to construct a simple model (template) whose outputs are sufficient to predict the behavior of the full model's (anchor's) outputs \cite{full1999templates}. 
Third, in Section \ref{subsec: FRS}, we develop an offline method to compute a gait parameterized forward reachable set that describes the evolution of the outputs of the simple model. 

Similar to recently developed work on motion planning for ground and aerial vehicles \cite{majumdar2017funnel,herbert2017fastrack,kousik2018bridging,kousik2019safe}, one can then require that all possible outputs in the forward reachable set satisfy the set of inequality constraints that we define that guarantee that the robot does not fall over during the $N$-steps. 
Unfortunately this type of set inclusion constraint can be challenging to enforce at run-time. 
Finally, in Sections \ref{subsec: set inclusion} and \ref{sec:mpc}, we describe how to incorporate this set-inclusion constraint as a set of inequality constraints in a Model Predictive Control (MPC) framework that are sufficient to ensure $N$-step walking that does not fall over.
Note, to simplify exposition, this paper focuses on an example implementation on a $14$-dimensional model of the robot RABBIT that is described in Section \ref{sec:prelim}. 
The remainder of this paper is organized as follows. 
Section \ref{sec:results} demonstrates the performance of the proposed approach on a walking example and Section \ref{sec:conclusion} concludes the paper.

\section{Preliminaries}
\label{sec:prelim}
This section introduces the notation, the dynamic model of the RABBIT robot, and a Simplified Biped Model (SBM) that is used throughout the remainder of this paper.
The following notation is adopted in this manuscript.
All sets are denoted using calligraphic capital letters.
Let $\R$ denote the set of real numbers, and let $\N_+$ denote the collection of all non-negative integers.
Give a set $\mathcal A \subset \R^n$ for some $n \in \N_+$, let $C^1(\mathcal A)$ denote the set of all differentiable continuous functions from $\mathcal A$ to $\R$ whose derivative is continuous and let $\lambda_\mathcal A$ denote the Lebesgue measure which is supported on $\mathcal A$.

\subsection{RABBIT Model (Anchor)}
\label{subsec:high_model}

This paper considers the walking motion of a planar 5-link model of RABBIT \cite{chevallereau2003rabbit}.
The walking motion of the RABBIT model consists of alternating phases of \emph{single stance} (one leg in contact with the ground) and \emph{double stance} (both legs in contact with the ground).
While in single stance, the leg in contact with the ground is called the \emph{stance leg}, and the non-stance leg is called the \emph{swing leg}.
The double stance phase is instantaneous.
The configuration of the robot at time $t$ is $q(t) := [q_h(t), q_v(t), q_1(t), q_2(t), q_3(t), q_4(t), q_5(t)]^\top \in \Q \subset \R^7$,
where $(q_h(t),q_v(t))$ are Cartesian position of the robot hip;
$q_1(t)$ is the torso angle relative to the upright direction;
$q_2(t)$ and $q_4(t)$ are the hip angles relative to stance and swing leg, respectively; and
$q_3(t)$ and $q_5(t)$ are the knee angles.
The joints $(q_2, q_3, q_4, q_5)$ are actuated, and $q_1$ is an underactuated degree of freedom.
Let $\theta (q) := q_1 + q_2 + q_3/2$ denote the \emph{stance leg angle}, and let $\phi(q) := q_1 + q_4 + q_5/2$ denote the \emph{swing leg angle}.
We refer to the configuration when the robot hip is right above the stance foot, i.e. $\theta=\pi$, as \emph{mid-stance}. 
We refer to the motion between the $i$-th and $(i+1)$-st swing leg foot touch down with the ground as the \emph{$i$-th step}.


Using the method of Lagrange, we can obtain a continuous dynamic model of the robot during swing phase:
\begin{equation}
\label{eq:dyn}
\dot{a}(t) = f( a(t), u(t) )
\end{equation}
where $a(t) = [q^\top(t),\dot{q}^\top(t)]^\top \in T\Q \subset \R^{14}$ denotes the tangent bundle of $\Q$, $u(t) \in U$, $U$ describes the permitted inputs to the system, and $t$ denotes time.
We model the RABBIT as a hybrid system and describe the instantaneous change using the notation of a \emph{guard} and a \emph{reset map}.
That is, suppose $(\theta(q(t)), c_\text{foot}(q(t)))$ denotes the stance leg angle and the vertical position of the swing foot relative to the stance foot, respectively, given a configuration $q(t)$ at time $t$.
The guard $\G$ is $\{ (b,b') \in T\Q \mid \pi/2 < \theta(b) < 3\pi/2, c_\text{foot}(b) = 0 \text{ and } \Dot{c}_\text{foot}(b,b') < 0 \}$. 
Notice the force of the ground contact imposes a holonomic constraint on stance foot position, which enables one to obtain a reset map: \cite[Section~3.4.2]{westervelt2007feedback}:
\begin{equation}
\label{eq:reset}
    \Dot{q}^+(t) = \Delta\big(\Dot{q}^-(t)\big),
\end{equation}
where $\Delta$ describes the relationship between the pre-impact and post impact velocities.
More details about the definition and derivation of this hybrid model can be found in \cite[Section~3.4]{westervelt2007feedback}.

To simplify exposition, this paper at run-time optimizes over a family of reference gaits that are characterized by their average velocity and step length. 
These reference gaits are described by a vector of \emph{control parameters} $P(i) =  \big(p_1(i), p_2(i)\big)\in\mathcal P$ for all $i \in \N$, where $p_1(i)$ denotes the average horizontal velocity and $p_2(i)$ denotes the step length between the $i$-th and $(i+1)$-st mid-stance position.
Note $\mathcal P$ is compact.
These reference gaits are generated by solving a finite family of nonlinear optimization problems using FROST in which we incorporate $p_1(i)$, $p_2(i),$ and periodicity as constraints, and minimize the average torque squared over the gait period  \cite{hereid2017frost}. 
Each of these problems yields a reference trajectory parameterized by $P(i)$ and interpolation is applied over these generated gaits to generate a continuum of gaits.
Given a control parameter, a control input into the RABBIT model is generated by tracking the corresponding reference trajectories using a classical PD controller.

Next, we define a solution to the hybrid model as a pair $(\I, a)$, where $\I = \{I_i\}_{i=0}^N$ is a \emph{hybrid time set} with $I_i$ being intervals in $\R$,  and $a = \{a_i(\cdot)\}_{i=0}^N$ is a finite sequence of functions with each element $a_i(\cdot): I_i \to T\Q$ satisfying the dynamics \eqref{eq:dyn} over $I_i$  where $N \in \N$ \cite[Definitions~3.3, 3.4, 3.5]{lygeros2012hybrid}. 
Denote each $I_i := [\tau_i^+, \tau_{i+1}^-]$ for all $i < N$.
$\tau_i$ corresponds to the time of the transition between $(i-1)$-th to $i$-th step.
We let $\tau_i^-$ correspond to the time just before the transition and and $\tau_i^+$ correspond to the time just after the transition.
Since transitions are assumed to be instantaneous, $\tau_i = \tau_i^- = \tau_i^+$ if all values exist.
When a transition never happens during the $i$-th step, we denote $\tau_{i-1}^- = +\infty$.
Note when $\tau_{i+1} < \infty$, $a_i(\tau_{i+1}^-) \in \G$ and $a_{i+1}(\tau_{i+1}^+) \in \Delta(a_i(\tau_{i+1}^-))$. 


\subsection{Simplified Biped Model (Template)}


As we show in Section \ref{sec:results}, performing online optimization with the full RABBIT model is intractable due to the size of its state space. 
In contrast, performing online optimization with the \emph{Simplified Biped Model} (\HIP{}) adopted from \cite{wight2008introduction} is tractable. 
This model consists of a point-mass $M$ and two mass-less legs each with a constant length $l$.
The configuration of the \HIP{} at time $t$ is described by the stance leg angle, $\hat{\theta}$, and the swing leg angle, $\hat{\phi}$.
The input into the model is the step length size and the guard is the set of configurations when $\hat{\theta} + \hat{\phi} = 2 \pi$.
The swing leg swings immediately to a specified step length.
During the swing phase, one can use the method of Lagrange to describe the evolution of the configuration as a function of the current configuration and the input. 
Subsequent to the instantaneous double stance phase, an impact with the ground happens with a coefficient of restitution of $0$.
We denote a hybrid execution of the \HIP{} as a pair $(\hat{\mathcal I}, \hat{a})$ where $\hat{\mathcal I} = \{\hat{I}_i\}_{i=0}^N$ is a hybrid time set with $\hat I_i := [\hat\tau_i^+, \hat\tau_{i+1}^-]$ and $\Hat{a} = \{\Hat{a}_i(\cdot)\}_{i=0}^N$ is a finite sequence of solutions to the \HIP{}'s equations of motion.

 \section{Outputs to Describe Successful Walking}
\label{sec:outputs}

During online optimization, we want to optimize over the space of parameterized inputs while introducing a constraint to guarantee that the robot does not fall over.
This section first formalizes what it means for the RABBIT model to walk successfully without falling over. 
Unfortunately due to the high-dimensionality of the RABBIT model, implementing this definition directly as a constraint during online optimization is intractable. 
To address this problem, in Section \ref{subsec:outputsrabbit} defines a set of outputs that are functions of the state of RABBIT and proves that the value of these outputs can determine whether RABBIT is able to walk successfully.
Subsequently in Section \ref{subsec:outputssimple} we define a corresponding set of outputs that are functions of the state of the \HIP{} and illustrate how their values can be used to determine whether RABBIT is able to walk successfully. 

To define successful walking on RABBIT, we begin by defining the time during step $i$ at which mid-stance occurs (i.e., the largest time $t$ at which $\theta(q(t)) = \pi$ during $I_i$) as
\begin{equation}
\label{eq:t_MS}
    t_i^{MS} := 
    \begin{cases}
        +\infty, & \hspace*{-2cm} \text{if } \theta(q(t)) < \pi \quad \forall t \in I_i, \\
        -\infty, & \hspace*{-2cm} \text{if } \theta(q(t)) > \pi \quad \forall t \in I_i, \\
        \max\{t \in I_i \mid \theta(q(t)) = \pi  \}, & \text{otherwise.}
    \end{cases}
\end{equation} 
Note if mid-stance is never reached during step $i$, then the mid-stance time is defined as plus or minus infinity depending upon if the hip-angle remains less than $\pi$ or greater than $\pi$ during step $i$, respectively. 
Using this definition, we formally define successful walking for the RABBIT model  as:
\begin{defn}
\label{def: success walking}
The RABBIT model \emph{walks successfully} in step $i \in \N$ if
\begin{enumerate}
    \item $t_i^{MS} \neq \pm \infty$, \label{def:succ:1}
    \item $\pi/2 < \theta(q(t)) < 3 \pi /2$ for all $t \in I_i$, and \label{def:succ:2}
    \item $\tau_{i+1}^- < +\infty$. \label{def:succ:3}
\end{enumerate}
\end{defn}
To understand this definition, note that the first requirement ensures that mid-stance is reached, the second requirement ensures that the hip remains above the ground, and the final requirement ensures that the swing leg actually makes contact with the ground. 
Though satisfying this definition ensures that RABBIT takes a step, enforcing this condition directly during optimization can be cumbersome due to the high dimensionality of the RABBIT dynamics.

\subsection{Outputs to Describe Successful RABBIT Walking}
\label{subsec:outputsrabbit}

This subsection defines a set of discrete outputs that are functions of the state of RABBIT model and illustrates how they can be used to predict failure.
We begin by defining another time variable $t_i^0$:
\begin{equation}
\label{eq:t_0}
    t_i^0 := 
    \begin{cases}
        \tau_i^+, & \hspace*{-3cm} \text{if } \dot{\theta}(q(t), \dot{q}(t)) < 0 \quad \forall t \in I_i, \\
        \tau_{i+1}^-, & \hspace*{-3cm} \text{if } \dot{\theta}(q(t), \dot{q}(t)) > 0 \quad \forall t \in I_i, \\
        \max\{t \in I_i \mid \dot{\theta}(q(t), \dot{q}(t)) = 0 \}, &
            \text{otherwise.}
    \end{cases}
\end{equation}
Note $t_i^0$ is defined to be the last time in $I_i$ when a sign change of $\dot{\theta}$ occurs; when a sign change does not occur, $t_i^0$ is defined as an endpoint of $I_i$ associated with the sign of $\dot{\theta}$.


\begin{figure}
  \centering 
  \includegraphics[width=\columnwidth,clip=true]{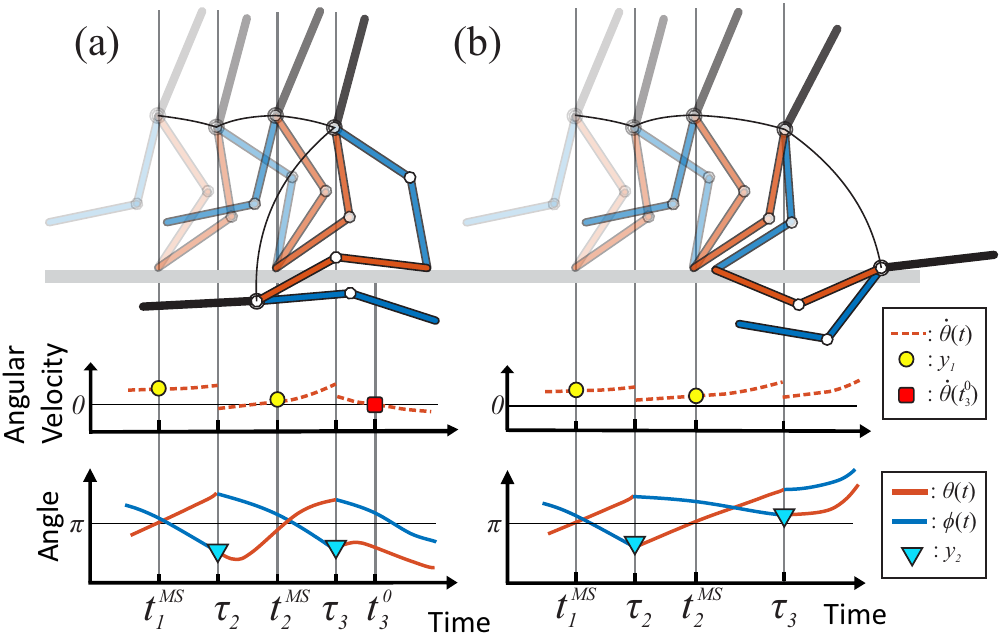}
  \caption{An illustration of how the values of the outputs can be used to determine whether the robot walks safely. 
  To ensure that the robot does not fall backwards, one can require that $y_1(i) \geq 0$ (left column).
  In particular if $y_1(i) < 0$, then $t_i^{MS} = +\infty$ which implies that the robot is falling backwards.
  To ensure that the robot does not fall forward, one can require that $y_2(i) \leq \pi$ (right column).}
  \label{fig:walk}
\end{figure}

We first define an output, $y_1: \N \to \R$ that can be used to ensure that $t_i^{MS} \neq +\infty$: 
\begin{equation}
    y_1(i) := 
    \begin{cases}
        \dot{\theta}(q(t_i^{MS}), \dot{q}(t_i^{MS})), &
        \text{if } t_i^{MS} \neq \pm \infty, \\
        -\sqrt{2 \mathrm{g} ( l_\text{st}(t_i^{0}) -q_v(t_i^{0}) )} / l_\text{st}(t_i^0), & \text{if } t_i^{MS} = +\infty, \\
        1 & \text{if } t_i^{MS} = -\infty,
    \end{cases}
    \label{def:y_1}
\end{equation}
where $\mathrm{g}$ is gravity and $l_\text{st}(t_i^{0})$ is the stance leg length at time $t_i^0$.
Note that $y_1(i)$ is the hip angular velocity when the mid-stance position is reached during the $i$-th step. 
When the mid-stance position is not reached, $-y_1(i)$ represents the additional hip angular velocity needed to reach the mid-stance position.
In particular, notice $t_i^{MS} \neq +\infty$ whenever $y_1(i) \geq 0$ .

Next, we define an output $y_2 : \N \to \R$ that can be used to ensure that $t_i^{MS} \neq -\infty$:
\begin{equation}
\label{def:y_2}
    y_2(i) := 
    \begin{cases}
        \phi(q(\tau_{i+1}^-)), & \text{ if } \tau_{i+1}^- < +\infty, \\
        2\pi, & \text{otherwise}.
    \end{cases}
\end{equation}
Note, $y_2(i)$ is the swing leg angle at touch-down at the end of the $i$-th step;
if touch-down does not occur, $y_2(i)$ is defined as $2\pi$.
Recall $\phi(q(\tau_{i+1}^-)) = \theta(q(\tau_{i+1}^+))$, so if $y_2(i) \leq \pi$, it then follows from \eqref{eq:t_MS} and \eqref{def:y_2} that $t_{i+1}^{MS} \neq -\infty$ and $\tau_{i+1}^-<+\infty$. 
Fig. \ref{fig:walk} illustrates the behavior of $y_1$ and $y_2$.

We now define our last two outputs $y_3,y_4:\N \to \R\cup\{-\infty,+\infty\}$ that can be used to ensure that the hip stays above the ground:
\begin{equation}
    y_3(i) :=\begin{cases}
    \inf \{\theta(q(t))\mid t\in [t_i^{MS}, t_{i+1}^{MS}]\}, & \text{if } t_{i+1}^{MS}, t_{i}^{MS}\in\R,\\
    -\infty, & \text{otherwise.}
    \end{cases}
\end{equation}
\begin{equation}
    y_4(i) :=\begin{cases}
    \sup \{\theta(q(t))\mid t\in [t_i^{MS}, t_{i+1}^{MS}]\}, & \text{if } t_{i+1}^{MS}, t_{i}^{MS}\in\R,\\
    +\infty, & \text{otherwise.}
    \end{cases}
\end{equation}
Finally, we let $\Y := \R \times \R \times \left(\R\cup\{-\infty,+\infty\}\right) \times \left(\R\cup\{-\infty,+\infty\}\right)$.
Using these definitions, we can prove the following theorem that constructs a sufficient condition to ensure successful walking by RABBIT.

\begin{thm}
\label{thm: walking successfully}
Suppose that the $0$-th step can be successfully completed (i.e. $\tau_0^+$ and $t_0^{MS}$ are finite, $\inf \{\theta(q(t))\mid t\in [\tau_0^+, t_0^{MS}]\}>\pi/2$, and $\sup \{\theta(q(t))\mid t\in [\tau_0^+, t_0^{MS}]\}<3\pi/2$)). 
Suppose $y_1(i)\geq 0,\, y_2(i)\leq\pi$, $y_3(i)>\pi/2$ and $y_4(i)<3\pi/2$ for each $i \in \{0,\cdots,N\}$, then the robot walks successfully at the $i$-th step for each $i \in \{0,\cdots,N\}$.
\end{thm}
\begin{proof}
\noindent  Notice $y_1(i) \geq 0 \Rightarrow t_i^{MS} \neq +\infty$ and $y_2(i) \leq \pi \Rightarrow  t_{i+1}^{MS} \neq -\infty$ for each $i \in \{1,\cdots,N\}$. 
By induction we have $t_i^{MS}$ is finite $\forall i \in \{1,\cdots,N\}$.
$y_2(i)\leq\pi<2\pi$ implies that $\tau_{i+1}^-<+\infty$.
By using the definitions of $y_3$ and $y_4$, one has that the robot walks successfully in the $i$-th step based on Definition \ref{def: success walking}.
\end{proof}

\subsection{Approximating Outputs Using the \HIP{}}
\label{subsec:outputssimple}

Finding an analytical expression describing the evolution of each of the outputs can be challenging. 
Instead we define corresponding outputs $\hat{y}(i):=\big(\hat{y}_1(i),\hat{y}_2(i),\hat{y}_3(i),\hat{y}_4(i)\big)\in\Y$ for \HIP{}.
Importantly, the dynamics of each of these corresponding outputs can be succinctly described.

As we did for the RABBIT model, consider the following set of definitions for the \HIP{}:
\begin{equation}
\label{eq:t_MS_low}
    \hat{t}_i^{MS} := 
    \begin{cases}
        +\infty, & \hspace*{-2cm} \text{if } \hat{\theta}(t) < \pi \quad \forall t \in\hat I_i, \\
        -\infty, & \hspace*{-2cm} \text{if } \hat{\theta}(t) > \pi \quad \forall t \in\hat I_i, \\
        \max\{t \in \hat{I}_i \mid \hat{\theta}(t) = \pi  \}, & \text{otherwise.}
    \end{cases}
\end{equation}

\begin{equation}
\label{eq:t_0_low}
    \hat{t}_i^0 := 
    \begin{cases}
        \hat{\tau}_i^+, & \hspace*{-2cm} \text{if } \dot{\hat{\theta}}(t) < 0 \quad \forall t \in \hat I_i, \\
        \hat{\tau}_{i+1}^-, & \hspace*{-2cm} \text{if } \dot{\hat{\theta}}(t) > 0 \quad \forall t \in \hat I_i, \\
        \max\{t \in \hat{I}_i \mid \dot{\hat{\theta}}(t) = 0 \}, &
            \text{otherwise.}
    \end{cases}
\end{equation}

\begin{equation}
    \hat{y}_1(i) := 
    \begin{cases}
       \dot{\hat{\theta}}(\hat{t}_i^{MS}), &
        \text{if } \hat{t}_i^{MS} \neq \pm \infty \\
        -\sqrt{2 \mathrm{g} ( l(1+\cos(\hat{\theta}(\hat{t}_i^0))) )} / l, & \text{if } \hat{t}_i^{MS} = +\infty \\
        1 & \text{if } \hat{t}_i^{MS} = -\infty,
    \end{cases}
\end{equation}

\begin{equation}
    \hat{y}_2(i) := 
    \begin{cases}
        \hat \phi(\hat{\tau}_{i+1}^-), & \text{if } \hat{\tau}_{i+1}^- < +\infty, \\
        2\pi, & \text{otherwise}.
    \end{cases}
\end{equation}

\begin{equation}
    \hat y_3(i) :=\begin{cases}
    \inf \{\hat \theta(t)\mid t\in [\hat t_i^{MS}, \hat t_{i+1}^{MS}]\}, & \text{if } \hat t_{i+1}^{MS}, \hat t_{i}^{MS}\in\R,\\
    -\infty, & \text{otherwise.}
    \end{cases}
\end{equation}
\begin{equation}
    \hat y_4(i) :=\begin{cases}
    \sup \{\hat \theta(t)\mid t\in [\hat t_i^{MS}, \hat t_{i+1}^{MS}]\}, & \text{if } \hat t_{i+1}^{MS}, \hat t_{i}^{MS}\in\R,\\
    +\infty, & \text{otherwise.}
    \end{cases}
\end{equation}
The discrete-time dynamics of each of these outputs of \HIP{} can be described by the following difference equations:
\begin{equation}
\begin{aligned}
    \hat{y}_1(i+1) &= f_{\hat{y}_1}\big( \hat{y}_1(i), P(i) \big) \\
    \hat{y}_2(i) &= f_{\hat{y}_2}\big( P(i) \big) \\
    \hat{y}_3(i) &= f_{\hat{y}_3}\big( \hat{y}_1(i), P(i) \big) \\
    \hat{y}_4(i) &= f_{\hat{y}_4}\big( \hat{y}_1(i), P(i) \big)
    \end{aligned}
    \label{eq:discrete-z} 
\end{equation}
for each $i \in \N$, $\hat{y}(i) \in \Y$, and $P(i) \in \PP$.
Such functions $f_{\hat{y}_1}$, $f_{\hat{y}_2}$,  $f_{\hat{y}_3}$ and $f_{\hat{y}_4}$ can be generated using elementary mechanics \footnote{A derivation can be found at: \url{https://github.com/pczhao/TA_GaitDesign/blob/master/SBM_dynamics.pdf}}.

To describe the gap between the discrete signals $y$ and $\hat{y}$ we make the following assumption:
\begin{assum}
\label{assum: modeling error bound}
    For any sequence of control parameters, $\{P(i)\}_{i \in N}$, and corresponding sequences of outputs, $\{y_1(i),y_2(i),y_3(i),y_4(i)\}_{i \in \N}$ and $\{\hat{y}_1(i),\hat{y}_2(i),\hat{y}_3(i),\hat{y}_4(i)\}_{i \in \N}$, generated by the RABBIT dynamics and \eqref{eq:discrete-z}, respectively, there  exists \emph{bounding functions} $\underline{B}_1$, $\overline{B}_1: \R   \times \PP \to \R$, $\overline{B}_2: \PP \times \R  \times \PP \to \R$, and $\underline{B}_3$, $\overline{B}_4: \R \times \PP \to \R$ satisfying
    \begin{gather}
        \underline{B}_1\big(y_1(i), P(i)\big) \leq  y_1(i+1) - \hat{y}_1(i+1)  \leq \overline{B}_1\big(y_1(i), P(i)\big) \label{eq:B1}\\
        y_2(i) - \hat{y}_2(i)\leq \overline{B}_2\big(P(i-1),y_1(i),P(i)\big)  \label{eq:B2}\\
        y_3(i) - \hat{y}_3(i)  \geq \underline{B}_3\big(y_1(i), P(i)\big) \label{eq:B3}\\
        y_4(i) - \hat{y}_4(i)  \leq \overline{B}_4\big(y_1(i), P(i)\big) \label{eq:B4}.
    \end{gather}
\end{assum}
In other words, if $y_1(i) = \hat{y}_1(i)$, then $\overline{B}_1$, $\underline{B}_1$, $\underline{B}_2$, $\underline{B}_3$, and $\overline{B}_4$ bound the maximum possible difference between $(y_1(i+1), y_2(i), y_3(i),y_4(i))$ and $(\hat{y}_1(i+1),\hat{y}_2(i), \hat{y}_3(i),\hat{y}_4(i))$.
Though we do not describe how to construct these bounding functions in this paper due to space limitations, one could apply SOS optimization to generate them \cite{smith2019}.
To simplify further exposition, we define the following:
\begin{equation}
\begin{split}
    \B(y_1(i),P(i)) :=
    [f_{\hat{y}_1}\big(y_1(i), P(i)\big) + \underline{B}_1\big(y_1(i), P(i)\big), \\
    f_{\hat{y}_1}\big(y_1(i), P(i)\big) + \overline{B}_1\big(y_1(i), P(i)\big)]
\end{split}
\end{equation}
for all $(y_1(i), P(i)) \in \R \times \PP$.
In particular, it follows from \eqref{eq:B1} that for any sequence of control parameters, $\{P(i)\}_{i \in N}$, and corresponding sequences of outputs, $\{y_1(i)\}_{i \in \N}$ generated by the RABBIT dynamics that $y_1(i+1) \in \B\big(y_1(i), P(i)\big)$ for all $i \in \N$.

\section{Enforcing N-Step Safe Walking}
\label{sec:problem}

This section proposes an online MPC framework to design a controller for the RABBIT model that can ensure successful walking for $N$-steps. 
In fact, when $N = 1$ one can directly apply Theorem \ref{thm: walking successfully} and Assumption \ref{assum: modeling error bound} to generate the following inequality constraints over $y_1(i)$, $P(i-1)$ and $P(i)$ to guarantee walking successfully from the $i$-th to the $(i+1)$-th mid-stance:
\begin{gather}
    f_{\hat y_1}\big(y_1(i), P(i)\big)
    +\underline{B}_1\big(y_1(i),P(i)\big)\geq 0, \label{cond: fall_back}\\
    f_{\hat y_2}\big( P(i) \big)
    +\overline{B}_2\big(P(i-1),y_1(i),P(i)\big)\leq\pi, \label{cond: fall_forward}\\
    f_{\hat y_3}\big(y_1(i), P(i)\big)
    +\underline{B}_3\big(y_1(i),P(i)\big)>\pi/2, \label{cond: hip_lower}\\
    f_{\hat y_4}\big(y_1(i), P(i)\big)
    +\overline{B}_4\big(y_1(i),P(i)\big)<3\pi/2. \label{cond: hip_upper}
\end{gather}

Unfortunately, to construct a similar set of constraints when $N > 1$, one has to either compute $(y_1(i),y_2(i),y_3(i),y_4(i))$ for each $i \leq N$, which can be computationally taxing, or one can apply \eqref{eq:B1} recursively to generate an outer approximation to $y_1(i)$ for each $i \leq N$ and then apply the remainder of Assumption \ref{assum: modeling error bound} to generate an outer approximation to $y_2(i), y_3(i),$ and $y_4(i)$ for each $i \leq N$.
In the latter instance, one would need the entire set of possible values for the outputs to satisfy the bounds described in \eqref{cond: fall_back}, \eqref{cond: fall_forward}, \eqref{cond: hip_lower}, and \eqref{cond: hip_upper} to ensure $N$-step safe walking.
This requires introducing a set inclusion constraint that can be cumbersome to enforce at run-time. 
To address these challenges, Section \ref{subsec: FRS} describes how to compute in an offline fashion, an $N$-step \emph{Forward Reachable Set} (FRS) that captures all possible outcomes for the outputs from a given initial state and set of control parameters for up to $N$ steps.
Subsequently, Section \ref{subsec: set inclusion} illustrates how to impose the set inclusion constraints as inequality constraints.

\subsection{Forward Reachable Set}
\label{subsec: FRS}
Letting $\Y_1 \subset \R$ be compact,
we define the \emph{$N$-step FRS of the output}:
\begin{defn}
\label{def:FRS}
    The $N$-step FRS of the output beginning from $\big(y_1(i), P(i)\big) \in \Y_1 \times \PP$ for $i \in \N$ and for $N \in \N$ is defined as
\begin{align}
    & {\W{}}_N\big(y_1(i), P(i)\big) := \bigcup_{n = i+1}^{i+N} \Big\{ y_1(n) \in  \Y_1 \mid \exists P(i+1), \ldots,\nonumber\\
    & \hspace{1.7cm} P(n-1) \in \PP \text{ such that } \forall j \in \{ i, \ldots, i + n - 1 \}, \nonumber\\
    & \hspace{1.7cm} \text{$y_1(j+1)$ is generated by the RABBIT} \nonumber \\
    & \hspace{1.7cm} \text{dynamics from $y_1(j)$ under $P(j)$ } \Big\} \label{eq:FRS}
\end{align}
\end{defn} 
In other words, given a fixed output $y_1(i)$ and the current control parameter $P(i)$, the FRS $\W_N$ captures all the outputs $y_1(j)$ that can be reached within $N$ steps, provided that all subsequent control parameters are contained in a set $\PP$.
The following result follows from the previous definition:
\begin{lem}
\label{lem: FRS inclusion}
\begin{equation}
        \W{}_M\big(y_1(i), P(i)\big) \subseteq \W{}_N\big(y_1(i), P(i)\big) \quad \forall 1 \leq M \leq N
        \label{FRS set inclusion}
    \end{equation}
\end{lem}


To compute an outer approximation of the FRS, one can solve the following infinite-dimensional linear problem over the space of functions:
\begin{align*}
    \inf_{w_N, v_1, \cdots, v_N} & \int_{\Y_1 \times \PP \times \Y_1} w_N(x_1, x_2, x_3) \, d\lambda_{\Y_1 \times \PP \times \Y_1} ~~\FRS{}\\
    \text{s.t.}~~~~ & v_1(x_1, x_2, x_3) \geq 0,\\
    &\hspace{3cm}\forall x_3 \in \B(x_1, x_2)\\
    &\hspace{3cm}\forall(x_1, x_2) \in \Y_1 \times \PP \\
    & v_{\zeta+1}(x_1, x_2, x_4) \geq v_\zeta(x_1, x_2, x_3), \\
    &\hspace{3cm}\forall \zeta \in \{ 1, 2, \cdots, N-1\} \\
    &\hspace{3cm} \forall x_4 \in \B(x_3, x_5)\\
    &\hspace{3cm}\forall(x_1, x_2, x_5) \in \Y_1 \times \PP \times \PP\\
    & w_N(x_1, x_2, x_3) \geq 0,\\
    &\hspace{3cm}\forall (x_1, x_2, x_3) \in \Y_1 \times \PP \times \Y_1 \\
    & w_N(x_1, x_2, x_3) \geq v_\zeta(x_1, x_2, x_3) + 1,\\
    &\hspace{3cm}\forall \zeta = 1, 2, \cdots, N\\
    &\hspace{3cm}\forall (x_1, x_2, x_3) \in \Y_1 \times \PP \times \Y_1  
\end{align*}
where the sets $\Y_1$ and $\PP$ are given, and the infimum is taken over an $(N+1)$-tuple of continuous functions
$(w_N, v_1, \cdots, v_N) \in \left( C^1(\Y_1 \times \PP \times \Y_1; \R) \right)^{N+1}$.
Note that only the \HIP{}'s dynamics appear in this program via $\mathcal B(\cdot,\cdot)$.

Next, we prove that the FRS is contained in the 1-superlevel set of all feasible $w$'s in \FRS{}:
\begin{lem}
\label{lem:FRS}
Let $(w_N, v_1, \cdots, v_N)$ be feasible functions to \FRS{}, then for all $\big( y_1(i), P(i) \big) \in \Y_1 \times \PP$
\begin{equation}
    \W{}_N\big(y_1(i), P(i)\big) \subseteq \left\{x_3 \in \Y_1 \mid w_N\big(y_1(i),P(i),x_3\big) \geq 1 \right\}.
\end{equation}

\end{lem}

\begin{proof}
Let $(w_N, v_1, \cdots, v_N)$ be feasible functions to \FRS{}.
Substitute an arbitrary $y_1(i)\in\Y_1$ and $P(i)\in\PP$ into $x_1$ and $x_2$, respectively.
Suppose $\mu\in \W_N\big(y(i), P(i)\big)$, then there exists a natural number $n \in [i+1, i+N]$ and a sequence of control parameters $P(i+1), \cdots, P(n-1) \in \PP$, such that  $y_1(j+1) \in \B\big(y_1(j), P(j)\big)$ for all $i \leq j \leq n-1$ and $\mu = y_1(n)$.

We prove the result by induction.
Let $x_3=y_1(i+1) \in \B\big(y_1(i), P(i)\big)$.
It then follows from the first constraint of \FRS{} that $v_1\big(y_1(i), P(i), y_1(i+1)\big) \geq 0$.
Now, suppose $v_\zeta\big(y_1(i), P(i), y_1(i+\zeta)) \geq 0$ for some $1 < \zeta \leq n-i-1$.
In the second constraint of \FRS{}, let $x_3=y_1(i+\zeta)$, $x_4=y_1(i+\zeta+1) \in \B\big(y_1(i+\zeta), P(i+\zeta)\big)$, and $x_5=P(i+\zeta) \in \PP'$, then $v_{\zeta+1}\big(y_1(i), P(i), y_1(i+\zeta+1)\big) \geq 0$.
By induction, we know $v_N\big(y_1(i), P(i), y_1(n)\big) \geq 0$.
Using the fourth constraint of \FRS{}, let $x_3=\mu=y_1(n)$, and we get $w_N\big(y_1(i), P(i),\mu\big) \geq 1$.
Therefore $\mu\in \left\{x_3 \in \Y_1 \mid w_N\big(y_1(i),P(i),x_3\big) \geq 1 \right\}$.
\end{proof}

Though we do not describe it here due to space restrictions, a feasible polynomial solution to \FRS{} can be computed offline by 
making compact approximation of $\Y_1$ and
applying Sums-of-Squares programming \cite{zhao2017control,mohan2017synthesizing}.

\subsection{Set Inclusion}
\label{subsec: set inclusion}

To ensure safe walking through $N$-steps beginning at step $i$, we require several set inclusions to be satisfied during online optimization.
First, we require that $\W{}_{N}\big(y_1(i),P(i)\big)\subseteq [0,\infty)$, which ensures that \eqref{cond: fall_back} is satisfied. 
Since we cannot compute $\W{}_{N}\big(y_1(i),P(i)\big)$ exactly we instead can require that the $1$-superlevel set of $w_N$ is a subset of $[0,\infty)$; however, this set inclusion is difficult to enforce using MPC. 
Instead we utilize the following theorem which follows as a result of the $S$-procedure technique described in Section 2.6.3 of \cite{boyd1994linear} and Lemma \ref{lem:FRS}:
\begin{thm}
\label{thm: set inclusion 1}
Let $(w_N, v_1, \cdots, v_N)$ be feasible functions to \FRS{} and $\W{}_N$ be as in Definition \ref{def:FRS}. 
Let $s_1,\,s_2:\Y_1\times\PP\times\Y_1\rightarrow\R$ be functions that are non-negative everywhere.
Suppose $\ell:\Y_1\times\PP\rightarrow\R$ satisfies the following inequality
\begin{align}
    &s_{1}(x_1,x_2,x_3)\cdot x_3-\ell(x_1,x_2)+\nonumber\\
    -&s_{2}(x_1,x_2,x_3)\cdot \big( w_N(x_1,x_2,x_3)-1 \big)\geq0
\end{align}
for every $(x_1,x_2,x_3)\in\Y_1\times\PP\times\Y_1$.
Then for any $y \in \Y_1,$ and $P \in \PP$, if $\ell\big(y_1,P\big)\geq0$, then $\W{}_{N}\big(y_1,P\big)\subseteq [0,\infty)$.
\end{thm}
\noindent Given a feasible solution to \FRS{}, one can construct polynomial functions $s_1,s_2,$ and $\ell$ offline that satisfy Theorem \ref{thm: set inclusion 1} using Sums-of-Squares programming \cite{zhao2017control,mohan2017synthesizing}.

Similarly we can utilize the following theorem to construct polynomial functions offline that allow us to verify whether safe $N$-step walking is feasible.
\begin{thm}
\label{thm:remainder set inclusions}
For each $\zeta \in \{1,2,\cdots,N-1\}$, suppose 
\begin{enumerate}
    \item \label{thm: set inclusion 2} $s^{y_2}_{\zeta,1},s^{y_2}_{\zeta,2}: \PP\times\Y_1\times\PP\rightarrow\R$ are functions that are non-negative everywhere and there exist functions $\ell^{y_2}_\zeta:\PP\times\PP\rightarrow\R$ that satisfy the following inequality
\begin{align}
    &s^{y_2}_{\zeta,1}(x_1,x_2,x_3)\cdot \big( \pi-f_{\hat y2}(x_3)-\overline{B}_2(x_1,x_2,x_3) \big) +\nonumber\\
    -&s^{y_2}_{\zeta,2}(x_1,x_2,x_3)\cdot x_2-\ell^{y_2}_\zeta(x_1,x_3)\geq0,
\end{align}
for every $(x_1,x_2,x_3)\in\PP\times\Y_1\times\PP$.
Then for each $\zeta \in \{1,2,\ldots,N-1\}$ and $P, P' \in \PP$, if $\ell^{y_2}_\zeta\big(P,P' \big)\geq0$ and $y_1 \in [0,\infty)$, then $f_{\hat y_2}\big( P' \big)
    +\overline{B}_2\big(P,y_1,P'\big)\leq\pi$.
    
    \item \label{thm: set inclusion 3} $s^{y_3}_{\zeta,1},s^{y_3}_{\zeta,2}: \Y_1\times\PP\rightarrow\R$ are functions that are non-negative everywhere, $\epsilon$ is a small positive number, and there exist functions $\ell^{y_3}_\zeta:\PP\rightarrow\R$ that satisfy the following inequality
\begin{align}
    & s^{y_3}_{\zeta,1}(x_1,x_2)\cdot\big( f_{\hat y_3}(x_1,x_2)+\underline{B}_3(x1,x2)-\pi/2-\epsilon \big) + \nonumber\\
    -& s^{y_3}_{\zeta,2}(x_1,x_2)\cdot x_1-\ell^{y_3}_\zeta(x_2)\geq0
\end{align}
for every $(x_1,x_2)\in\Y_1\times\PP$.
Then for each $\zeta \in \{1,2,\ldots,N-1\}$ and $P\in\PP$, if  $\ell^{y_3}_\zeta\big(P \big)\geq0$ and $y_1 \in [0,\infty)$, then $f_{\hat y_3}\big(y_1, P\big) +\underline{B}_3\big(y_1,P\big)>\pi/2$.
\item  \label{thm: set inclusion 4} $s^{y_4}_{\zeta,1},s^{y_4}_{\zeta,2}: \Y_1\times\PP\rightarrow\R$ are functions that are non-negative everywhere, $\epsilon$ is a small positive number, and there exists $\ell^{y_4}_\zeta:\PP\rightarrow\R$ that satisfy the following inequality
\begin{align}
    & s^{y_4}_{\zeta,1}(x_1,x_2)\cdot \big( 3\pi/2-\epsilon - f_{\hat y_4}(x_1,x_2) - \overline{B}_4(x1,x2)\big) + \nonumber\\
    -& s^{y_4}_{\zeta,2}(x_1,x_2)\cdot x_1-\ell^{y_4}_\zeta(x_2)\geq0
\end{align}
for every $(x_1,x_2)\in\Y_1\times\PP$.
Then for each $\zeta \in \{1,2,\ldots,N-1\}$ and $P\in\PP$, if $\ell^{y_4}_\zeta\big(P(i+\zeta) \big)\geq0$ and $y_1 \in [0,\infty)$, then $f_{\hat y_4}\big(y_1, P\big)
    +\overline{B}_4\big(y_1,P\big)<3\pi/2$.
\end{enumerate}
\end{thm}

\noindent One can construct polynomial functions $\ell, \ell_\zeta^{y_2},\ell_\zeta^{y_3},$ and $\ell_\zeta^{y_4}$ offline that satisfy Theorem \ref{thm: set inclusion 1} using Sums-of-Squares programming \cite{zhao2017control,mohan2017synthesizing}. 
As we describe next, these functions allow us to represent the set inclusions conditions as inequality constraints that are amenable to online optimization.


\section{Model Predictive Control Problem}
\label{sec:mpc}

We use a MPC framework to select a gait parameter for RABBIT by solving the following nonlinear program:
\begin{align*}
	\min_{\substack{P(i)\\\vdots\\P(i+N-1)}} &  r\left(y(i),P(i),P(i+1),\cdots,P(i+N-1)\right)  \hspace*{0.7cm}(\text{OL})\\
	\text{s.t.} \hspace*{0.4cm}&\ell\big(y_1(i),P(i)\big)\geq0,\\
	& f_{\hat y2}\big( P(i) \big)
    +\overline{B}_2\big(P(i-1),y_1(i),P(i)\big)\leq\pi,\\
	&f_{\hat y3}\big( y_1(i), P(i) \big)
    +\underline{B}_3\big(y_1(i),P(i)\big)>\pi/2,\\
    &f_{\hat y4}\big( y_1(i), P(i) \big)
    +\overline{B}_4\big(y_1(i),P(i)\big)<3\pi/2,\\
    &\ell^{y_2}_\zeta\big(P(i+\zeta-1),P(i+\zeta) \big)\geq0, ~~~~\forall \zeta=1,\cdots,N-1,\\
    &\ell^{y_3}_\zeta\big(P(i+\zeta) \big)\geq0, \hspace{2.2cm}\forall \zeta=1,\cdots,N-1,\\
    &\ell^{y_4}_\zeta\big(P(i+\zeta) \big)\geq0, \hspace{2.2cm}\forall \zeta=1,\cdots,N-1,\\
	& P(i),P(i+1),\cdots,P(i+N-1)\in\PP
\end{align*}
where $r\in C^1(\Y\times\mathcal P^N; \R)$ is any user specified cost function.
Notice that (OL) is solved at the $i$-th mid-stance
and only the optimal $P(i)$ is applied to the RABBIT and the problem is then solved again for the $(i+1)$-st step. 
The constraints of (OL) together with Theorems \ref{thm: set inclusion 1} and \ref{thm:remainder set inclusions} lead to the following theorem:
\begin{thm}
Suppose that RABBIT is at the $i$-th mid-stance, then tracking the gait parameters associated with any feasible solution to $(OL)$ ensures that RABBIT can walk successfully for the next $N$-steps.
\end{thm}
\section{Results}
\label{sec:results}
We evaluate our method on $300$ simulation trials in which the robot is required to track a random desired speed.
Our MATLAB implementation of these pair of experiments can be found at: \url{https://github.com/pczhao/TA_GaitDesign.git}.
In both experiments we set $N=3$.
The space of control parameter is restricted to be $\mathcal P=[0.25,2]\times[0.15,0.7]$.

We compare our method with a na\"ive method and the direct method. 
The na\"ive method uses the \HIP{} model to design a gait in an MPC framework.
The direct method uses the full-order dynamics of the RABBIT model to design a controller by solving an optimal control problem via FROST \cite{hereid2017frost}.
To simplify the comparison, each method performs optimization by minimizing the difference between a user specified speed and the speed of the model used during optimization. 

To apply our method, we begin by using the commercial solver MOSEK to compute an outer approximation to the \FRS{} and each of the  the $\ell$-functions in (OL) on a machine with 144 64-bit 2.40GHz Intel Xeon CPUs and 1 Terabyte memory.
Note bounding functions that satisfy Assumption \ref{assum: modeling error bound} are constructed through simulation and the ones used during our implementation can be found in the aforementioned repository.

\begin{figure}
  \centering  
  \includegraphics[trim =0in 0in 0in 0.in,width=0.85\columnwidth,clip=true]{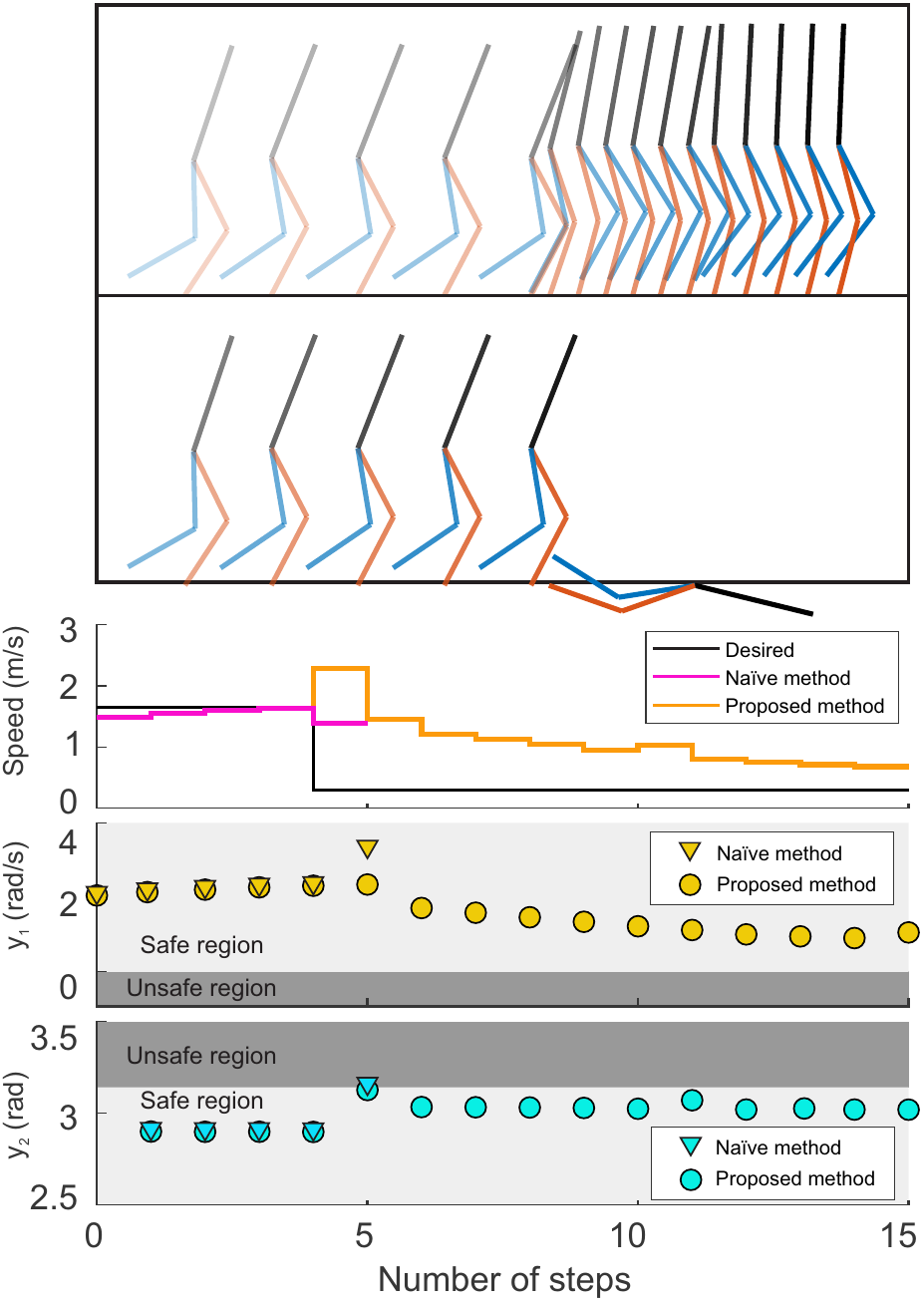}
  \caption{An illustration of the performance of the method proposed in this paper (top) and the \naive{} method (second from top). 
  Note that the rapid change in the desired speed (third from top) generates a gait which cannot be tracked by just considering a \HIP{} model without additional constraints.
  By ensuring that the outputs satisfy the inequality constraints proposed in Theorem \ref{thm: walking successfully} (bottom two sub-figures), the proposed method is able to safely track the synthesized gaits. 
  Note the \naive{} method violates the $y_2$ constraint proposed in Theorem \ref{thm: walking successfully} on Step 5.}
  \label{fig:PD1}
\end{figure}

Fig. \ref{fig:PD1} illustrates the performance of the na\"ive method and the method proposed in this paper on a sample trial. 
Note in particular that the gait generated by the na\"ive method is unable to be followed by the full-order RABBIT model. 
On the other hand, as shown in Fig. \ref{fig:PD1}, the method proposed in this paper is able to generate a gait that can satisfy the safety requirements described in Theorem \ref{thm: walking successfully}. 
This results in a controller which can track the synthesized gait without falling over.


Across all $300$ trials the computation time of the na\"ive method is $0.01$ seconds, the direct method is $93.12$ seconds, and the proposed method is $0.11$ seconds. 
Moreover, the RABBIT model falls 2$\%$ of the time with the na\"ive method, but never falls with the proposed method or the direct method.

\section{Conclusion}
\label{sec:conclusion}
This paper develops a method to generate safety-preserving controllers for full-order (anchor) models by performing reachability analysis on simpler (template) models while bounding the modeling error. 
The method is illustrated on a 5-link, $14$-dimenstional RABBIT model, and is shown to allow the robot to walk safely while utilizing controllers designed in a real-time fashion. 

Though this method enables real-time motion planning, future work will consider several extensions that will enable real-world robotic control.
First, a template and associated outputs need to be constructed for 3D motion. 
Second, no guarantee is provided that the optimization problem (OL) solved at each step in the MPC will return a feasible solution.


\renewcommand{\bibfont}{\normalfont\small}
{\renewcommand{\markboth}[2]{}
\printbibliography}

\end{document}